\documentclass[11pt]{article}
\usepackage{amsmath,amsfonts,amsthm,mathrsfs}
\usepackage{amssymb}
\usepackage{color}
\usepackage[dvips]{graphicx}
\usepackage{epsfig}
\usepackage{epsf}
\usepackage{float}
\usepackage{subfigure}
\newtheorem{example}{Example}
\newtheorem{theorem}{Theorem}
\newtheorem{lemma}[theorem]{Lemma}
\newtheorem{proposition}[theorem]{Proposition}
\newtheorem{remark}[theorem]{Remark}
\newtheorem{corollary}[theorem]{Corollary}
\newtheorem{definition}[theorem]{Definition}

\def\dmin{\displaystyle\min}
\def\dmax{\displaystyle\max}
\def\dsum{\displaystyle\sum }
\def\dsup{\displaystyle\sup }
\def\dprod{\displaystyle\prod }

\def\diint{\displaystyle\iint }
\def\begeqn{\begin{equation}}
\def\endeqn{\end{equation}}
\def\begth{\begin{theorem}}
\def\endth{\end{theorem}}
\def\begprop{\begin{proposition}}
\def\endprop{\end{proposition}}
\def\begcor{\begin{corollary}}
\def\endcor{\end{corollary}}
\def\begdef{\begin{definition}}
\def\enddef{\end{definition}}
\def\beglemm{\begin{lemma}}
\def\endlemm{\end{lemma}}
\def\begexm{\begin{example}}
\def\endexm{\end{example}}
\def\begrem{\begin{remark}}
\def\endrem{\end{remark}}
\def\beg{\begin}
\def\gd{\delta}

\def\gl{\lambda}

\def\gs{\sigma}

\def\gO{\Omega}

\def\bz{{\bf z}}

\def\O{\mathcal{O}}

\def\S{\mathbb{S}}
\def\N{\mathbb{N}}
\def\R{\mathbb{R}}
\def\X{\mathcal{X}}
\def\Y{\mathcal{Y}}
\def\Z{\mathcal{Z}}
\def\E{\mathcal{E}}

\def\C{\mathcal{C}}

\def\E{\mathcal{E}}
\def\F{\mathcal{F}}
\def\EX{{\mathbb{E}}}

\def\beg{\begin}
\def\Tr{\hbox{\bf Tr}}

\def\gl{\lambda}
\def\R{\mathbb{R}}

\def\T{\mathcal{T}}

\def\wE{\widetilde{\mathcal{E}}}
\def\wR{\widetilde{R}}
\def\wX{\widetilde{X}}
\def\wM{\widetilde{M}}
\def\wb{\widetilde{b}}
\def\lE{\overline{\mathcal{E}}}

  \title{Generalisation Bounds for Metric and Similarity  Learning
\footnote{Corresponding author: Yiming Ying.  ~Email: y.ying@exeter.ac.uk}}

\bigskip
\bigskip
\bigskip

\author{Qiong Cao,  Zheng-Chu Guo and Yiming Ying\\\\
College   of Engineering, Mathematics and Physical Sciences\\  University of Exeter, Harrison Building, EX4 4QF,  UK}

\date{}

\begin{document}

\maketitle

\begin{abstract}
Recently, metric learning and similarity learning have attracted a large amount of interest. Many models and optimisation algorithms have been proposed. However, there is relatively little work on the generalisation analysis of such methods. In this paper, we derive novel generalisation bounds of metric and similarity learning.  In particular, we first show that the generalisation analysis reduces to the estimation of the Rademacher average over ``sums-of-i.i.d." sample-blocks related to the specific matrix norm.  Then, we derive generalisation bounds for metric/similarity learning with different matrix-norm regularisers by estimating their specific Rademacher complexities.  Our analysis indicates that  sparse metric/similarity learning with $L^1$-norm regularisation could  lead to significantly better  bounds than those with Frobenius-norm regularisation. Our novel generalisation analysis develops and refines the techniques of U-statistics and Rademacher complexity analysis.
\end{abstract}

\section{Introduction}
The success of many machine learning algorithms (e.g. the nearest neighborhood classification
and k-means clustering) depends on the concepts of distance metric and similarity. For instance, k-nearest-neighbor (kNN) classifier depends on a distance function to identify the nearest neighbors for classification; k-means algorithms depend on the pairwise distance measurements between examples for clustering.  Kernel methods and information retrieval methods rely on a similarity measure between samples.  Many existing studies have been devoted to learning a metric or similarity automatically from data, which is usually referred to as {\em metric learning} and {\em similarity learning}, respectively.

Most work in metric learning focuses on learning a (squared) Mahalanobis distance defined, for any $x,t\in \R^d$, by $d_M(x,t) = {(x-t)M(x-t)^\top}$ where $M$ is a positive semi-definite matrix, see e.g. \cite{Bar,Davis,GR,Gold,Shen,Weinberger,Xing,Yang,Ying}.  Concurrently, the pairwise similarity defined by $s_M(x,t) = xMt^\top$ was studied in \cite{Chechik,Kar,Maurer,Chechik2}. These methods have been successfully applied to to  various real-world problems including  information retrieval and face verification \cite{Chechik,Gui,Hoi,Ying-jmlr}.
Although there are a large number of studies  devoted to supervised metric/similarity learning based on different objective functions, few studies address the generalisation analysis of such methods.   The recent work \cite{Jin} pioneered the generalisation analysis for metric learning using the concept of uniform stability \cite{Bous}.  However, this approach only works for the strongly   convex norm, e.g. the Frobenius norm,   and the  offset term is fixed which makes the generalisation analysis essentially different.

In this paper,  we develop a novel approach for generalisation analysis of metric learning and similarity learning which can deal with general matrix regularisation terms including Frobenius norm  \cite{Jin}, sparse $L^1$-norm \cite{Rosales}, mixed $(2,1)$-norm \cite{Ying} and trace-norm \cite{Ying,Shen}. In particular, we first show that the generalisation analysis for metric/similarity learning reduces to the estimation of the Rademacher average over ``sums-of-i.i.d." sample-blocks related to the specific matrix norm, which we refer to as the {\em Rademacher complexity for metric (similarity) learning}.  Then, we  show how to estimate the Rademacher complexities with different matrix regularisers.  Our analysis indicates that sparse metric/similarity learning with   $L^1$-norm regularisation could lead to significantly better generalisation bounds than that with Frobenius norm regularisation, especially when the dimension of the input data is high. This is nicely consistent with the rationale that sparse methods are more effective for  high-dimensional data analysis.
Our novel generalisation analysis develops and extends Rademacher complexity analysis \cite{BM,KPan} to the setting of metric/similarity learning by using techniques of U-statistics \cite{Clem,Gine}.

The paper is organized as follows.  The next section reviews the models of metric/similarity learning. Section \ref{sec:bounds} establishes the main theorems. In Section \ref{sec:exm}, we derive and discuss generalisation bounds for metric/similarity learning with various matrix-norm regularisation terms. Section \ref{sec:conclusion} concludes the paper.

\medskip

\noindent {\bf Notation:}  Let $\N_n = \{1,2,\ldots, n\}$
for any $n\in \N$. For any $X,Y\in \R^{d\times n}$,
$\langle X, Y\rangle = \Tr(X^\top Y)$ where $\Tr(\cdot)$ denotes
the trace  of a matrix. The space of symmetric $d$ times $d$ matrices
will be denoted by $\S^d.$  We equip $\S^d$ with a general matrix norm $\|\cdot\|$; it can be a Frobenius norm, trace-norm and mixed norm. Its associated dual norm is denoted, for any $M\in \S^d$, by $\|M\|_\ast = \sup\{ \langle X, M \rangle: X\in \S^d, \|X\|\le 1 \}.$  The Frobenius norm on matrices or vector is always denoted by $\|\cdot\|_F.$ Later on we use the conventional notation that $X_{ij}=(x_i-x_j)(x_i-x_j)^\top$ and $\widetilde{X}_{ij} = x_ix_j^\top.$

\section{Metric/Similarity Learning Formulation}\label{sec:model}
In our learning setting, we have an input space
$\X\subseteq\R^d$ and an output (label) space $\Y$. Denote  $\Z =
\X\times \Y$ and suppose $\bz := \{z_i = (x_i,y_i)\in\Z: i\in \N_n\}$ an
i.i.d. training set according to an unknown distribution $\rho$ on
$\Z.$ Denote  the $d\times n$ input data matrix by ${\bf X} = (x_i:
i\in\N_n)$ and the $d\times d$ {\em distance matrix} by $M= (M_{\ell
k})_{\ell,k\in\N_d}$. Then, the
(pseudo-) distance between $x_i$ and $x_j$ is measured by
$$d_M(x_i,x_j) = (x_i-x_j)^\top M(x_i-x_j).$$
The goal of metric learning is   to  identify a distance function
$d_M(x_i,x_j)$ such that it yields  a small value for  a similar
pair  and a large value for  a dissimilar pair. The  bilinear similarity function is defined by
$$s_M(x_i,x_j) = x_i^\top  M x_j.$$  Similarly, the target of similarity learning is to learn $M\in \S^d$ such that
it reports a large similarity value for a similar pair and  a small similarity value for a dissimilar pair.  It is worth pointing out that we do not require the  positive semi-definiteness of the matrix $M$ throughout this paper. However,  we do assume $M$ to be symmetric, since this will guarantee the distance (similarity) between $x_i$ and $x_j$  ($d_M(x_i,x_j)$) is equivalent to that  between $x_j$ and $x_i$ ($d_M(x_j,x_i)$).

There are two main terms in the  metric/similarity  learning
model: {\em empirical error} and {\em matrix regularisation term}.
The empirical error function is to employ the similarity and dissimilarity information
provided by the label information  and the appropriate matrix
regularisation term is to avoid overfitting and improve
generalisation performance.

For any pair of
samples $(x_i,x_j)$, let $r(y_i,y_j) = 1$ if $y_i = y_j$ otherwise
$r(y_i,y_{j})=-1$. It is expected that there exists an offset term $b\in \R$ such that $d_M(x_i,x_j)\le b$ for   $r(y_i,y_j) = 1$ and   $d_M(x_i,x_j)>b$ otherwise. This naturally leads to the empirical error \cite{Jin} defined by
$${1\over n(n-1)} \sum_{i,j\in \N_n,i\neq j}I[r(y_i,y_j)(d_M(x_i,x_j)-b)>0]$$
where the indicator function $I[x]$ equal $1$ if $x$ is true and zero
otherwise.

Due to the indicator function, the above empirical error is not convex which is difficult to do
optimisation.  A usual way to overcome this shortcoming is to  upper-bound it with a
convex loss function. For instance,  we can use the
the hinge loss to upper-bound the indicator function  which
leads to the following empirical error: \begeqn\label{eq:err-obj}
\mathcal{E}_\bz(M,b) := {1\over n(n-1)} \sum_{i,j\in \N_n,i\neq j}[1+r(y_i,y_j)(d_M(x_i,x_j)-b)]_+
\endeqn

In order to avoid overfitting,  we need to enforce a  regularisation term denoted by $\|M\|$, which will restrict the complexity of the distance matrix.   We emphasize here $\|\cdot\|$ denotes a general matrix norm in the linear space $\S^d$.  Putting
the regularisation term and the empirical error term together yields the following metric learning model:
\begeqn\label{eq:model} (M_\bz,b_\bz) = \arg\min_{M\in \S^d,b\in\R} \bigl\{\mathcal{E}_\bz(M,b) +\gl \|M\|^2\bigr\},
\endeqn
where $\gl>0$ is a trade-off parameter.

Different regularisation terms lead to different metric learning formulations. For instance, the Frobenius norm $\|M\|_F$ is used in \cite{Jin}.  To favor the element-sparsity, \cite{Rosales} introduced the $L^1$-norm regularisation $\|M\| = \sum_{\ell,k\in\N_d}|M_{\ell k}|.$  \cite{Ying} proposed the mixed $(2,1)$-norm $ \|M\| = \sum_{\ell\in\N_d}\bigl(\sum_{k\in\N_d}|M_{\ell k}|^2\bigr)^{1\over 2}$ to encourage the column-wise sparsity of the distance matrix. The trace-norm regularisation $\|M\| = \sum_{\ell}\gs_\ell(M)$ was also considered by \cite{Ying,Shen}. Here,  $\{\sigma_\ell: \ell\in \N_d\}$ denote the singular values of a matrix $M\in \S^d.$ Since $M$ is symmetric, the singular values of $M$ are identical to the absolute values of its eigenvalues.

In analogy to the formulation of metric learning, we consider the following empirical error for   similarity learning \cite{Maurer,Chechik}:
\begeqn\label{eq:err-obj}
\widetilde{\mathcal{E}}_\bz(M,b) := {1\over n(n-1)} \sum_{i,j\in \N_n,i\neq j}[1 -r(y_i,y_j)(s_M(x_i,x_j)-b)]_+.
\endeqn
This leads to the regularised formulation for similarity learning defined as follows:
\begeqn\label{eq:sim-model} (\widetilde{M}_\bz,\widetilde{b}_\bz) = \arg\min_{M\in \S^d,b\in\R} \bigl\{\widetilde{\mathcal{E}}_\bz(M,b) +\gl \|M\|^2\bigr\}.
\endeqn
\cite{Maurer} used the Frobenius-norm regularisation for similarity learning.
 The trace-norm regularisation has been used by \cite{Chechik2}   to encourage a low-rank similarity matrix $M.$

\section{Statistical Generalisation Analysis}\label{sec:bounds}
In this section, we mainly give a detailed proof of generalisation bounds for metric and similarity learning.
In particular,  we develop a novel line of generalisation analysis for metric and similarity
learning with general matrix regularisation terms.    The key observation is that the empirical data term $\E_\bz(M,b)$ for metric learning is a modification of U-statistics and it is expected to converge to its expected form defined by
\begeqn\label{eq:trueerror}\E(M,b) =  \diint
(1+r(y,y')(d_M(x,x')-b))_+d\rho(x,y)d\rho(x',y').\endeqn The empirical term $\wE_\bz(M,b)$ for similarity learning is expected to converge to
\begeqn\label{eq:sim-trueerror}\wE(M,b) =  \diint
(1-r(y,y')(s_M(x,x')-b))_+d\rho(x,y)d\rho(x',y').\endeqn
The target of generalisation analysis   is to  bound
the true error $\E(M_\bz,b_\bz)$ by the empirical error $\E_\bz(M_\bz,b_\bz)$ for metric learning and $\wE(\wM_\bz,\wb_\bz)$ by the empirical error $\wE_\bz(\wM_\bz,\wb_\bz)$ for similarity learning.

In the sequel, we provide a detailed proof for generalisation bounds of metric learning.  Since the proof  for similarity learning is  exactly the same as that for metric learning, we only mention the results followed with some brief  comments.

\subsection{Bounding the Solutions}

By the definition of $(M_\bz, b_\bz)$,  we know that
$$ \E_\bz(M_\bz,b_\bz) + \gl \|M_\bz\|^2 \le  \E_\bz(0,0) + \gl \|0\| = 1 $$
which implies that
    \begeqn\label{eq:M-bnd}
\|M_\bz\|\le {1\over \sqrt{\gl}}.
\endeqn
Now we turn our attention to deriving the bound of the offset term $b_\bz$ by modifying the
techniques  in \cite{Chen} which was originally  developed to estimate  the offset term of the soft-margin SVM.
\begin{lemma}\label{lemm:b-bnd} For any samples $\bz$ and $\gl>0$, there exists a minimizer $(M_\bz,b_\bz)$ of problem (\ref{eq:model}) such that \begeqn\label{eq:min-max-bnd} \min_{i\neq
j} [d_{M_\bz}(x_i,x_j) -b_\bz]\le 1, ~~  \max_{i\neq j}
[d_{M_\bz}(x_i,x_j) -b_\bz]\ge -1.\endeqn
\end{lemma}
\begin{proof}
Firstly  we prove the inequality $\min_{i\neq j} [d_{M_\bz}(x_i,x_j) -b_\bz]\le 1.$ To this end,  we first consider the special case where the training set $\bz$ only contains two examples $z_1=(x_i,y_1)$ and $z_2=(x_2,y_2)$ with distinct labels, i.e. $y_1\neq y_2$. For any $\lambda>0,$  let $(M_\bz, b_\bz)=(\mathbf{0},-1) $,  and observe that $ \E_\bz(\mathbf{0},-1) + \gl \|\mathbf{0}\|^2 = 0.$  This observation implies that $(M_\bz, b_\bz)$ is a minimizer of problem (\ref{eq:model}).  Consequently, we have the desired result since $\min_{i\neq
j} [d_{M_\bz}(x_i,x_j) -b_\bz]=d_{M_\bz}(x_1,x_2) -b_\bz=1.$

Now let us consider the general case where the training set $\bz$ has at least two examples with the same label. In this case, we prove the inequality by contradiction. Suppose that $r = \dmin_{i\neq j} [d_{M_\bz}(x_i,x_j)-b_\bz]>1$ which equivalently implies that $d_{M_\bz}(x_i,x_j) -(b_\bz+r-1)\ge 1$ for any
$i\neq j.$ Hence, for any $i\neq j$ and any pair of examples $(x_i,x_j)$ with distinct labels, i.e. $y_i\neq y_j$ (equivalently $r(y_i,y_j)=-1$), there holds $$\bigl(1+ r(y_i,y_j)(d_{M_\bz}(x_i,x_j)
-b_\bz-r+1)\bigr)_+ =\bigl(1-(d_{M_\bz}(x_i,x_j)
-b_\bz-r+1)\bigr)_+ = 0.$$ Consequently,
$$\begin{array}{ll} \E_{\bz}(M_\bz, b_\bz+r-1) &  = {1\over n(n-1)} \dsum_{i\neq j} \Big(1+r(i,j)(d_{M_\bz}(x_i,x_j)
-b_\bz-r+1)\Big)_+\\
& ={1\over n(n-1)} \dsum_{i\neq j,y_i= y_j} (1+d_{M_\bz}(x_i,x_j)
-b_\bz-(r-1))_+\\ & < {1\over n(n-1)} \dsum_{i\neq j,y_i= y_j}
(1+d_{M_\bz}(x_i,x_j) -b_\bz)_+\le \E_{\bz}(M_\bz, b_\bz).
\end{array}$$
The above estimation implies that $\E_{\bz}(M_\bz, b_\bz+r-1) + \gl \|M_\bz\| < \E_{\bz}(M_\bz,
b_\bz)+ \gl\|M_\bz\|$ which contradicts the definition of the minimizer
$(M_\bz,b_\bz)$. Hence, $ r = \dmin_{i\neq j} [d_{M_\bz}(x_i,x_j)
-b_\bz]\le 1$.

Secondly,  we prove the inequality $\dmax_{i\neq j}[d_{M_\bz}(x_i,x_j) -b_\bz]\ge -1$ in analogy to the above argument.
Consider a special case where the training set $\bz$ contains only two examples $z_1=(x_1,y_i)$ and $z_2=(x_2,y_2)$ with
the same label, i.e. $y_1=y_2.$ For any given $\lambda>0,$ let $(M_\bz,b_\bz)= (\mathbf{0},1).$  Since $ \E_\bz(\mathbf{0},1) + \gl \|\mathbf{0}\|^2 = 0$, $(\mathbf{0},1)$ is a minimizer of problem (\ref{eq:model}). The desired estimation follows from the fact that $\dmax_{i\neq j} d_{M_\bz}(x_i, x_j)-b_\bz =0-1= -1.$

Now let us consider the general case where the training set $\bz$ has at least two examples with distinct labels. We prove the estimation by contradiction. Assume $r=\dmax_{i\neq j} d_{M_\bz}(x_i, x_j)-b_\bz < -1,$ then $d_{M_\bz}(x_i,x_j) -(b_\bz+r+1)\leq -1$ holds for any $i\neq j.$ This implies, for any pair of examples  $(x_i,x_j)$ with the same label, i.e. $r(i,j)=1$,  that $\Big(1+r(i,j)(d_{M_\bz}(x_i,x_j)
-b_\bz-r-1)\Big)_+=0.$  Hence,
$$\begin{array}{ll} \E_{\bz}(M_\bz, b_\bz+r+1) &  = {1\over n(n-1)} \dsum_{i\neq j} \Big(1+r(i,j)(d_{M_\bz}(x_i,x_j)
-b_\bz-r-1)\Big)_+ \\ &={1\over n(n-1)}\dsum_{i\neq j,y_i\ne y_j} \Big(1-d_{M_\bz}(x_i,x_j)
+b_\bz+(r+1)\Big)_+ \\ & < {1\over n(n-1)} \dsum_{i\neq j,y_i \ne y_j}
(1-d_{M_\bz}(x_i,x_j) +b_\bz)_+\le \E_{\bz}(M_\bz, b_\bz).
\end{array}$$
The above estimation yields that $\E_{\bz}(M_\bz, b_\bz+r+1)+\lambda\|M_\bz\|^2<\E_{\bz}(M_\bz, b_\bz)+\lambda\|M_\bz\|^2$ which contradicts the definition of the minimizer $(M_\bz, b_\bz).$  Hence, we have the desired inequality $\dmax_{i\neq j} d_{M_\bz}(x_i, x_j)-b_\bz \ge -1$ which
completes the proof of the lemma.
\end{proof}
\begin{corollary}\label{offsetbound}
For any samples $\bz$ and $\gl>0$, there exists a minimizer $(M_\bz,b_\bz)$ of problem (\ref{eq:model}) such that
\begeqn\label{eq:b-bnd} |b_\bz|\le 1+  \bigl(\max_{i\neq j }
\|X_{ij}\|_\ast\bigr)
 \|M_\bz\|.\endeqn
\end{corollary}
\begin{proof}
Recall that $X_{ij} = (x_i-x_j)(x_i- x_j)^\top$ and observe, by the definition of the dual norm $\|\cdot\|_\ast$, that $$  d_M(x_i,x_j) = \langle X_{ij}, M \rangle \le \|X_{ij}\|_\ast \|M\|.$$
Using the above observation, estimation (\ref{eq:b-bnd}) follows directly from inequality (\ref{eq:min-max-bnd}). This completes the proof.
\end{proof}
Denote
\begeqn \label{eq:bound}\F = \Bigl\{(M,b):  \|M\| \le {1/\sqrt{\gl}}, ~~ |b| \le 1+  X_\ast  \|M\|\Bigr\}, \endeqn
where $$X_\ast =   \sup_{x,x'\in \X }
\|(x-x')(x-x')^\top\|_\ast. $$
From the above corollary, for any samples $\bz$  we can easily see that the optimal solution $(M_\bz, b_\bz)$ of formulation (\ref{eq:model}) belongs to the bounded set $\F \subseteq \S^d \times \R.$

We end this subsection with two remarks. Firstly, in what follows, we restrict our attention to the minimizer $(M_\bz, b_\bz)$  of formulation (\ref{eq:model}) which satisfies inequality (\ref{eq:b-bnd}). Secondly, our formulation (\ref{eq:model}) for metric learning focused on the hinge loss which is widely used in the community of metric learning, see e.g \cite{Jin,Weinberger,Ying-jmlr}. Similar results to those in the above corollary can easily be obtained for $q$-norm loss given, for any $x\in\R$, by $(1-x)_+^q$ with $q>1.$ However, it still remains a question to us on how to estimate the term $b$ for general loss functions.

\subsection{Generalisation Bounds}

Before stating the generalisation bounds, we introduce some
notations. For any $z=(x,y), z' = (x',y')\in \Z$, let $\Phi_{M,b} (z, z') = (1+
r(y,y')(d_{M}(x,x') -b))_+.$  Hence, for any
$(M,b)\in \F$, \begeqn\label{eq:Phi-bnd}
\sup_{z,z'}\sup_{(M,b)\in\F}\Phi_{M,b}(z,z') \le B_\gl  : = 2\bigl(1+  X_\ast /\sqrt{\gl}\bigr).\endeqn
Let $ \lfloor{n\over 2}\rfloor$ denote the largest integer less than ${n\over 2}$ and recall the definition that $X_{ij} = (x_i-x_j)(x_i-x_j)^\top.$ We now define Rademacher average over sums-of-i.i.d. sample-blocks related to the dual matrix
norm $\|\cdot\|_\ast$ by
\begeqn\label{eq:rad-emp}   \widehat{R}_n  = {1\over \lfloor{n\over
2}\rfloor} \EX_\gs\Bigl\| \sum_{i=1}^{\lfloor{n\over
2}\rfloor}\gs_i X_{i ({\lfloor{n\over
2}\rfloor+i})}\Bigr\|_\ast, \endeqn and its expectation is denoted by
$ R_n = \EX_\bz \bigl[\widehat{R}_n\bigr].$ Our main theorem below shows that the generalisation bounds for metric learning critically depend on the quantity of  $ R_n$.    For this reason, we refer to $R_n$ as the {\em Radmemcher complexity for metric learning}. It is worth mentioning that metric learning formulation (\ref{eq:model}) depends on the norm $\|\cdot\|$ of the linear space $ \S^d$ and the Rademacher complexity $R_n$ is related to its dual norm $\|\cdot\|_\ast$.

\begin{theorem}\label{thm:main}   Let $(M_\bz, b_\bz)$ be the solution of
formulation (\ref{eq:model}). Then, for any $0<\gd<1$,  with probability $1-\gd$  we have that
\begeqn\label{eq:gen-bound}\begin{array}{ll}  \E(M_\bz,b_\bz) -
\E_\bz(M_\bz,b_\bz) & \le \dsup_{(M,b)\in \F} \Bigl[\E(M,b) -
\E_\bz(M,b)\Bigr]  \\ & \le
 {4R_n \over \sqrt{\gl}} +  {4(3+ 2X_\ast /\sqrt{\gl}) \over \sqrt{n}} + 2\bigl(1+   X_\ast /\sqrt{\gl}\bigr)\left({2\ln \bigl({1\over \gd}\bigr)\over  {n}}\right)^{1\over 2}.\end{array}\endeqn
\end{theorem}

\begin{proof}  The proof of the theorem can be divided into three steps
as follows.

{\bf\em Step 1}: ~ Let $\EX_\bz$ denote the expectation with respect
to samples $\bz$. Observe that $\E(M_\bz,b_\bz) -
\E_\bz(M_\bz,b_\bz) \le \dsup_{(M,b)\in \F} \Bigl[\E(M,b) -
\E_\bz(M,b)\Bigr].$ For any ${\bf z} = (z_1,\ldots,z_{k-1},z_k,z_{k+1},\ldots,
z_n)$ and $ {\bf z}'=(z_1,\ldots, z_{k-1},z'_k, z_{k+1},\ldots, z_n)$  we know from inequality (\ref{eq:Phi-bnd}) that
$$\begin{array}{ll}& \Bigl|\dsup_{(M,b)\in \F} \Bigl[\E(M,b) -
\E_\bz(M,b)\Bigr] -\dsup_{(M,b)\in \F} \Bigl[\E(M,b) -
\E_{\bz'}(M,b)\Bigr] \Bigr|\\ & \le \dsup_{(M,b)\in \F}|\E_\bz(M,b)
-  \E_{\bz'}(M,b)| \\ &={1\over n(n-1)}\dsup_{(M,b)\in
\F}\dsum_{j\in \N_n,j\neq
k}|\Phi_{M,b}(z_k,z_j) - \Phi_{M,b}(z'_k,z_j)|\\
&\le  {1\over n(n-1)}\dsup_{(M,b)\in \F}\dsum_{j\in \N_n,j\neq
k}|\Phi_{M,b}(z_k,z_j)|+|\Phi_{M,b}(z'_k,z_j)| \\ & \le 4\bigl(1+   X_\ast /\sqrt{\gl}\bigr)/n.
\end{array}$$
Applying McDiarmid's inequality \cite{McD} (see Lemma \ref{lem:McD}
in the Appendix) to the term $\dsup_{(M,b)\in \F} \Bigl[\E(M,b) -
\E_\bz(M,b)\Bigr]$, with probability $1-{\gd}$ there holds
\begeqn\begin{array}{ll}\label{eq:Mcd-1} \dsup_{(M,b)\in \F}
\Bigl[\E(M,b) - \E_\bz(M,b)\Bigr]  & \le \EX_\bz\dsup_{(M,b)\in \F}
\Bigl[\E(M,b) - \E_\bz(M,b)\Bigr] \\ & ~~ +    2\bigl(1+   X_\ast /\sqrt{\gl}\bigr)\left({2\ln \bigl({1\over \gd}\bigr)\over  n}\right)^{1\over 2}.
\end{array}\endeqn

Now we only need to estimate the first term in the
expectation form on the right-hand side of the above equation by
symmetrization techniques.

{\bf\em  Step 2}: ~ To estimate $\EX_\bz\dsup_{(M,b)\in \F}
\Bigl[\E(M,b) - \E_\bz(M,b)\Bigr]$, applying Lemma \ref{lem:U-stat-lem} with $q_{(M,b)} (z_i,z_j) =\E(M,b) -  (1+ r(y_i,y_j)(d_M(x_i,x_j)-b))_+$ implies that
 \begeqn\label{eq:key-inequality}\EX_\bz\dsup_{(M,b)\in \F}
\Bigl[\E(M,b) - \E_\bz(M,b)\Bigr]\le \EX_\bz\dsup_{(M,b)\in \F}
\Bigl[\E(M,b) - {\lE}_\bz(M,b)\Bigr],\endeqn where
$\overline{\E}_\bz(M,b) = {1\over \lfloor{n\over 2}\rfloor}
\dsum_{i=1}^{\lfloor{n\over
2}\rfloor}\Phi_{M,b}(z_{i},z_{\lfloor{n\over 2}\rfloor+i}).$ Now let
$\bar{\bz} = \{\bar{z}_1,\bar{z}_2,\ldots, \bar{z}_{n}\}$ be i.i.d. samples which are independent of $\bz$,
then
\begeqn\label{eq:key-eq2}\begin{array}{ll}\EX_\bz\dsup_{(M,b)\in \F}
\Bigl[\E(M,b) - \overline{\E}_\bz(M,b)\Bigr] &  =
\EX_{\bz}\dsup_{(M,b)\in \F} \Bigl[ \EX_{\bar{\bz}}
\bigl[ \ \overline{\E}_{\bar{\bz}}(M,b)\bigr] -
\overline{\E}_\bz(M,b)\Bigr] \\
& \le  \EX_{\bz,\bar{\bz}}\dsup_{(M,b)\in \F} \Bigl[ \
\overline{\E}_{\bar{\bz}}(M,b) - \overline{\E}_\bz(M,b)\Bigr]
\end{array}\endeqn
By standard symmetrization techniques (see e.g. \cite{BM}),
for i.i.d. Rademacher variables $\{\gs_i \in\{\pm 1\}:
i\in\N_{\lfloor{n\over 2}\rfloor}\}$, we have that
\begeqn\label{eq:sym-inequ}\begin{array}{ll} & \EX_{\bz,\bar{\bz}}\dsup_{(M,b)\in \F}
\Bigl[  {\lE}_{\bar{\bz}}(M,b) -
 {\lE}_\bz(M,b)\Bigr] \\ & =\EX_{\bz,\bar{\bz}} {1\over
\lfloor{n\over 2}\rfloor}\dsup_{(M,b)\in \F}
\dsum_{i=1}^{\lfloor{n\over
2}\rfloor}\gs_i\Bigl[\Phi_{M,b}(\bar{z}_{i},\bar{z}_{\lfloor{n\over
2}\rfloor+i}) - \Phi_{M,b}({z}_{i},{z}_{\lfloor{n\over
2}\rfloor+i})\Bigr] \\
& = 2\EX_{\bz,\gs}{1\over \lfloor{n\over
2}\rfloor}\dsup_{(M,b)\in \F} \dsum_{i=1}^{\lfloor{n\over
2}\rfloor}\gs_i\Phi_{M,b}(\bar{z}_{i},\bar{z}_{\lfloor{n\over
2}\rfloor+i})\\  & \le 2\EX_{\bz,\gs}{1\over
\lfloor{n\over 2}\rfloor}\dsup_{(M,b)\in \F}\Bigl|
\dsum_{i=1}^{\lfloor{n\over
2}\rfloor}\gs_i\Phi_{M,b}(\bar{z}_{i},\bar{z}_{\lfloor{n\over
2}\rfloor+i})\Bigr|.\end{array} \endeqn Applying the
contraction property of Rademacher averages (see Lemma
\ref{lem:contr-prop} in the Appendix) with $\Psi_i(t) =  \bigl(1+ r(y_i, y_{\lfloor{n\over
2}\rfloor+i}) t \bigr)_+ - 1$, we have the following estimation for the last term on the
righthand side of the above inequality:
\begeqn\label{eq:inter2}\begin{array}{ll} & \EX_{\gs}{1\over \lfloor{n\over 2}\rfloor}\dsup_{(M,b)\in \F}\Bigl|
\dsum_{i=1}^{\lfloor{n\over
2}\rfloor}\gs_i\Phi_{M,b}(\bar{z}_{i},\bar{z}_{\lfloor{n\over
2}\rfloor+i})\Bigr| \\ & \le \EX_{\gs} {1\over
\lfloor{n\over 2}\rfloor} \dsup_{(M,b)\in \F}\Bigl|
\dsum_{i=1}^{\lfloor{n\over
2}\rfloor}\gs_i (\Phi_{M,b}(\bar{z}_{i},\bar{z}_{\lfloor{n\over
2}\rfloor+i}) -1 )\Bigr| + {1\over
\lfloor{n\over 2}\rfloor } \EX_{\gs}\Bigl|\dsum_{i=1}^{\lfloor{n\over 2}\rfloor}\gs_i \Bigr| \\
& \le  {2\over \lfloor{n\over
2}\rfloor} \EX_{\gs} \dsup_{(M,b)\in \F}\Bigl|
\dsum_{i=1}^{\lfloor{n\over 2}\rfloor}\gs_i \bigl(
d_M(x_i,x_{\lfloor{n\over 2}\rfloor+i}) -b \bigr)\Bigr|
 +
{1\over \lfloor{n\over
2}\rfloor} \EX_{\gs}\Bigl|\dsum_{i=1}^{\lfloor{n\over 2}\rfloor}\gs_i\Bigr| \\
& \le  {2\over \lfloor{n\over
2}\rfloor} \EX_{\gs} \dsup_{\|M\|\le {1\over \sqrt{\gl}}}\Bigl|
\dsum_{i=1}^{\lfloor{n\over 2}\rfloor}\gs_i
d_M(x_i,x_{\lfloor{n\over 2}\rfloor+i})\Bigr|
 +
{(3+ 2X_\ast /\sqrt{\gl})\over \lfloor{n\over
2}\rfloor} \EX_{\gs}\Bigl|\dsum_{i=1}^{\lfloor{n\over 2}\rfloor}\gs_i\Bigr|
\end{array}\endeqn

{\bf\em Step 3 }: It remains to  estimate the  terms on the righthand side
of inequality (\ref{eq:inter2}). To this end, observe that
$$\EX_{\gs}\Bigl|\dsum_{i=1}^{\lfloor{n\over 2}\rfloor}\gs_i\Bigr|
\le \Bigr( \EX_{\gs}\Bigl|\dsum_{i=1}^{\lfloor{n\over
2}\rfloor}\gs_i\Bigr|^2\Bigl)^{1\over 2}\le \sqrt{\lfloor{n\over
2}\rfloor}.$$ Moreover,
$$\begin{array}{ll}\EX_\gs\dsup_{\|M\|\le
{1\over \sqrt{\gl}}}\Bigl| \dsum_{i=1}^{\lfloor{n\over 2}\rfloor}\gs_i
d_M(x_i,x_{\lfloor{n\over 2}\rfloor+i})\Bigr|  & =
\EX_\gs\dsup_{\|M\|\le {1\over \sqrt{\gl}}}\Bigl|\langle
\dsum_{i=1}^{\lfloor{n\over 2}\rfloor}\gs_i (x_i-x_{\lfloor{n\over
2}\rfloor+i})(x_i-x_{\lfloor{n\over
2}\rfloor+i})^\top, M\rangle\Bigr| \\
& \le {1\over \sqrt{\gl}}\EX_\gs\Bigl\| \sum_{i=1}^{\lfloor{n\over
2}\rfloor}\gs_i X_{i {(\lfloor{n\over
2}\rfloor+i})}\Bigr\|_\ast.
\end{array}$$
Putting the above estimations and  inequalities (\ref{eq:sym-inequ}),
(\ref{eq:inter2}) together yields that
$$\EX_{\bz,\bar{\bz}}\dsup_{(M,b)\in \F} \Bigl[
{\lE}_{\bar{\bz}}(M,b) - {\lE}_\bz(M,b)\Bigr] \le
 {2(3+ 2X_\ast /\sqrt{\gl}) \over \sqrt{\lfloor{n\over
2}\rfloor  }} +  {4R_n \over \sqrt{\gl}} \le {4(3+ X_\ast /\sqrt{\gl}) \over \sqrt{n}}
+ {2R_n \over \sqrt{\gl}}.$$
Consequently, combining this with inequalities
(\ref{eq:key-inequality}), (\ref{eq:key-eq2}) implies that
$$\EX_\bz\dsup_{(M,b)\in \F}
\Bigl[\E(M,b) - \E_\bz(M,b)\Bigr]\le {4(3+ 2X_\ast /\sqrt{\gl}) \over \sqrt{n}}
+ {4R_n \over \sqrt{\gl}}. $$ Putting this estimation with (\ref{eq:Mcd-1})
completes the proof the theorem.
\end{proof}

In the setting of similarity learning,  $X_\ast$  and $R_n$   are replaced by
\begeqn\label{eq:sim-Rn} \widetilde{X}_\ast =  \sup_{x,t\in\X} \|x t^\top\|_\ast ~~~ \hbox{ and } ~~ \widetilde{R}_n = {1\over \lfloor{n\over
2}\rfloor}\EX_\bz \EX_\gs\Bigl\| \sum_{i=1}^{\lfloor{n\over
2}\rfloor}\gs_i \widetilde{X}_{i ({\lfloor{n\over
2}\rfloor+i})}\Bigr\|_\ast, \endeqn
where $  \widetilde{X}_{i ({\lfloor{n\over
2}\rfloor+i})} = x_i x^\top_{\lfloor{n\over
2}\rfloor+i}.$ Let $\widetilde{\F} =  \Bigl\{(M,b):  \|M\| \le {1/\sqrt{\gl}}, ~~ |b| \le 1+  \wX_\ast  \|M\|\Bigr\}$.   Using  the exactly same argument as above, we can prove the following bound for similarity learning formulation (\ref{eq:sim-model}).
\begin{theorem}\label{thm:sim-main}   Let $(\wM_\bz, \wb_\bz)$ be the solution of
formulation (\ref{eq:sim-model}). Then, for any $0<\gd<1$,  with probability $1-\gd$  we have that
\begeqn\label{eq:sim-gen-bound}\begin{array}{ll}  \wE(\wM_\bz,\wb_\bz) -
\wE_\bz(\wM_\bz,\wb_\bz) & \le \dsup_{(M,b)\in \widetilde{\F}} \Bigl[\wE(M,b) -
\wE_\bz(M,b)\Bigr]  \\ & \le
 {4\widetilde{R}_n \over \sqrt{\gl}} +  {4(3+ 2\widetilde{X}_\ast /\sqrt{\gl}) \over \sqrt{n}} + 2\bigl(1+   \widetilde{X}_\ast /\sqrt{\gl}\bigr)\left({2\ln \bigl({1\over \gd}\bigr)\over  {n}}\right)^{1\over 2}.\end{array}\endeqn
\end{theorem}

\section{Estimation of $R_n$ and Discussion}\label{sec:exm}
From Theorem \ref{thm:main}, we need to estimate the Rademacher average for metric learning, i.e. $R_n $,  and the quantity $X_\ast$ for different matrix regularisation terms. Without loss of generality, we only focus on popular matrix norms such as the Frobenius norm \cite{Jin}, $L^1$-norm \cite{Rosales}, trace-norm \cite{Ying,Shen} and mixed $(2,1)$-norm \cite{Ying}.

\begin{example}[Frobenius norm]\label{exm:fro}
Let the matrix norm be the Frobenius norm i.e. $\|M \| = \|M\|_F$,  then the quantity $X_\ast = \sup_{x,x\in \X} \|x-x'\|^2_F $ and the Rademacher complexity is estimated as follows: $$R_n \le  {2 X_\ast \over \sqrt{n }}  = {2\sup_{x,x'\in \X} \|x-x'\|^2_F \over \sqrt{n}}.$$  Let $(M_\bz, b_\bz)$ be a solution of formulation (\ref{eq:model}) with Frobenius norm regularisation.  For any $0<\gd<1$, with probability $1-\gd$ there holds
  \begeqn\label{eq:fro-bnd} \begin{array}{ll}\E(M_\bz,b_\bz) -
\E_\bz(M_\bz,b_\bz)   & \le
  2\Big(1+   {\sup_{x,x\in \X} \|x-x'\|^2_F \over \sqrt{\gl}}\Big)\sqrt{{2\ln \bigl({1\over \gd}\bigr)\over  n}} \\  & \quad ~ +  {  16\sup_{x,x'\in \X} \|x-x'\|^2_F \over \sqrt{n\gl}}  +  { 12 \over \sqrt{n}}. \end{array}\endeqn
\end{example}
\begin{proof}  Note that the dual norm of the Frobenius norm is itself. The estimation of $X_\ast$ is straightforward. The Rademacher complexity $R_n$ is estimated as follows:
$$\begin{array}{ll}R_n  & = {1\over \lfloor{n\over
2}\rfloor } \EX \left( \sum_{i,j=1}^{\lfloor{n\over
2}\rfloor }  \gs_i\gs_j \langle x_i - x_{\lfloor{n\over
2}\rfloor+i},  x_j - x_{\lfloor{n\over
2}\rfloor+j}\rangle^2   \right)^{1\over 2} \\ & \le  {1\over \lfloor{n\over
2}\rfloor } \EX_\bz \left( \EX_\gs \sum_{i,j=1}^{\lfloor{n\over
2}\rfloor }  \gs_i\gs_j \langle x_i - x_{\lfloor{n\over
2}\rfloor+i},  x_j - x_{\lfloor{n\over
2}\rfloor+j}\rangle^2   \right)^{1\over 2} \\
& = {1\over \lfloor{n\over
2}\rfloor } \EX_\bz \left(  \sum_{i=1}^{\lfloor{n\over
2}\rfloor }  \|  x_i - x_{\lfloor{n\over
2}\rfloor+i}\|_F^4  \right)^{1\over 2} \\ & \le {X_\ast \big /\sqrt{\lfloor{n\over
2}\rfloor }} \le {2 X_\ast \over \sqrt{n }}.\end{array}$$
Putting this estimation back into equation (\ref{eq:gen-bound}) completes the proof of Example \ref{exm:fro}.
\end{proof}

Other popular matrix norms for metric learning are the $L^1$-norm, trace-norm and mixed $(2,1)$-norm.  The dual norms are respectively $L^\infty$-norm, spectral norm (i.e. the maximum of singular values) and mixed $(2,\infty)$-norm.  All these  dual norms mentioned above are less than the Frobenius norm. Hence,  the following estimation always holds true for all the norms mentioned above: $$X_\ast \le \sup_{x,x\in \X} \|x-x'\|^2_F, ~~ \hbox{ and }~~ R_n \le  {2\sup_{x,x'\in \X} \|x-x'\|^2_F \over \sqrt{n}}.$$
Consequently, the generalisation bound (\ref{eq:fro-bnd}) holds true for metric learning formulation (\ref{eq:model}) with  $L^1$-norm, or trace-norm or mixed $(2,1)$-norm regularisation. However, in some cases, the above upper-bounds are too conservative.  For instance, in the following examples we can show that more refined estimation of  $R_n$ can be obtained by applying the Khinchin inequalities for Rademacher averages \cite{Gine}.

\begin{example}[Sparse $L^1$-norm] \label{exm:L1} Let the matrix norm be the $L^1$-norm i.e. $\|M \| = \sum_{\ell,k\in \N_d} |M_{\ell k}|$.  Then,  $X_\ast     =  \sup_{x,x'\in \X }  \|x-x'\|_\infty^2$ and $$ R_n \le   4 \sup_{x,x'\in \X }  \|x-x'\|_\infty^2  \sqrt{e \log d \over n}.$$ Let $(M_\bz, b_\bz)$ be a solution of formulation (\ref{eq:model}) with $L^1$-norm regularisation.  For any $0<\gd<1$, with probability $1-\gd$ there holds
\begeqn\label{eq:L1-bnd} \begin{array}{ll}\E(M_\bz,b_\bz) -
\E_\bz(M_\bz,b_\bz)   & \le
  2\Big(1+   {\sup_{x,x\in \X} \|x-x'\|^2_\infty \over \sqrt{\gl}}\Big)\sqrt{{2\ln \bigl({1\over \gd}\bigr)\over  n}} \\  & \quad ~ +  {  8\sup_{x,x'\in \X} \|x-x'\|^2_\infty (1 + 2\sqrt{e\log d}) \over \sqrt{n\gl}}  +  { 12 \over \sqrt{n}}. \end{array}\endeqn
\end{example}

\begin{proof} The dual norm of the $L^1$-norm is the $L^\infty$-norm.   Hence,
$ X_\ast  =  \sup_{x,x'\in \X }  \|x-x'\|_\infty^2.$ To estimate $R_n$, we observe, for any $1<q<\infty$, that
\begeqn\label{eq:inter} \begin{array}{ll}   {R}_n & =  {1\over \lfloor{n\over
2}\rfloor} \EX_\bz\EX_\gs\Bigl\| \sum_{i=1}^{\lfloor{n\over
2}\rfloor}\gs_i X_{i ({\lfloor{n\over
2}\rfloor+i})}\Bigr\|_\infty \le    {1\over \lfloor{n\over
2}\rfloor} \EX_\bz\EX_\gs\Bigl\| \sum_{i=1}^{\lfloor{n\over
2}\rfloor}\gs_i X_{i ({\lfloor{n\over
2}\rfloor+i})}\Bigr\|_q \\ &  :  ={1\over \lfloor{n\over
2}\rfloor}
\EX_\bz\EX_\gs \left( \sum_{\ell, k \in \N_d}  \bigl| \sum_{i=1}^{\lfloor{n\over
2}\rfloor}\gs_i (x^k_i- x^k_{{\lfloor{n\over
2}\rfloor}+i})(x^\ell_i- x^\ell_{{\lfloor{n\over
2}\rfloor}+i}) \bigr|^q \right)^{1\over q}\\
& \le  {1\over \lfloor{n\over
2}\rfloor}\EX_\bz \left( \sum_{\ell, k \in \N_d} \EX_\gs \bigl| \sum_{i=1}^{\lfloor{n\over
2}\rfloor} \gs_i(x^k_i- x^k_{{\lfloor{n\over
2}\rfloor}+i})(x^\ell_i- x^\ell_{{\lfloor{n\over
2}\rfloor}+i})  \bigr|^q \right)^{1\over q}
 \end{array}\endeqn
where $x_i^k$ represents the $k$-th coordinate element of vector $x_i\in \R^d.$
To estimate the term on the right-hand side of inequality (\ref{eq:inter}), we apply  the Khinchin-Kahane inequality (See Lemma \ref{lem:khin} in the Appendix) with $p=2<q<\infty$ yields that
\begeqn \begin{array}{ll} & \EX_\gs\bigl| \sum_{i=1}^{\lfloor{n\over
2}\rfloor} \gs_i(x^k_i- x^k_{{\lfloor{n\over
2}\rfloor}+i})(x^\ell_i- x^\ell_{{\lfloor{n\over
2}\rfloor}+i})  \bigr|^q  \\ & \le q^{q\over 2} \bigl(\EX_\gs\bigl| \sum_{i=1}^{\lfloor{n\over
2}\rfloor} \gs_i (x^k_i- x^k_{{\lfloor{n\over
2}\rfloor}+i})(x^\ell_i- x^\ell_{{\lfloor{n\over
2}\rfloor}+i})  \bigr|^2\bigr)^{q\over 2} \\
&  = q^{q\over 2}  \bigl(\sum_{i=1}^{\lfloor{n\over
2}\rfloor}(x^k_i- x^k_{{\lfloor{n\over
2}\rfloor}+i})^2 (x^\ell_i- x^\ell_{{\lfloor{n\over
2}\rfloor}+i})^2\bigr)^{q\over 2}  \le \dsup_{x,x'\in \X }  \|x-x'\|_\infty^{2q} ( \lfloor{n\over
2}\rfloor)^{q\over 2}  q^{q\over 2}.
\end{array} \endeqn
Putting the above estimation back into (\ref{eq:inter}) and letting $q= 4\log d $ implies that
$$\begin{array}{ll} {R}_n  & \le   \dsup_{x,x'\in \X }  \|x-x'\|_\infty^{2}  d^{2 \over q} \sqrt{q} \big/ \sqrt{ \lfloor{n\over
2}\rfloor}  = 2 \dsup_{x,x'\in \X }  \|x-x'\|_\infty^{2}\sqrt{ e \log d \big / \lfloor{n\over
2}\rfloor} \\ & \le 4 \dsup_{x,x'\in \X }  \|x-x'\|_\infty^{2}\sqrt{ e \log d \big / n}.\end{array}$$
Putting the estimation for $X_\ast$ and $R_n$ into Theorem \ref{eq:gen-bound} yields inequality (\ref{eq:L1-bnd}).  This completes the proof of Example \ref{exm:L1}.
\end{proof}

\begin{example}[Mixed $(2,1)$-norm] \label{exm:21-norm} Consider $\|M \| = \sum_{\ell\in \N_d} \sqrt{\sum_{k\in \N_d}|M_{\ell k}|^2}.$  Then,  we have $X_\ast     =  \bigl[\sup_{x,x'\in \X }  \|x-x'\|_F\bigr]\bigl[\sup_{x,x'\in \X }  \|x-x'\|_\infty\bigr],$ and $$ R_n \le   4 \bigl[ \sup_{x,x'\in \X }  \|x-x'\|_\infty\bigr]\bigl[\sup_{x,x'\in \X }  \|x-x'\|_F\bigr]  \sqrt{e \log d \over n}.$$ Let $(M_\bz, b_\bz)$ be a solution of formulation (\ref{eq:model}) with mixed $(2,1)$-norm.  For any $0<\gd<1$, with probability $1-\gd$ there holds
\begeqn\label{eq:L21-bnd} \begin{array}{ll}\E(M_\bz,b_\bz) -
\E_\bz(M_\bz,b_\bz)   & \le
  2\Big(1+   {\bigl[ \sup_{x,x'\in \X }  \|x-x'\|_\infty\bigr]\bigl[\sup_{x,x'\in \X }  \|x-x'\|_F\bigr] \over \sqrt{\gl}}\Big)\sqrt{{2\ln \bigl({1\over \gd}\bigr)\over  n}} \\  & \quad ~ +  {  8\bigl[ \sup_{x,x'\in \X }  \|x-x'\|_\infty\bigr]\bigl[\sup_{x,x'\in \X }  \|x-x'\|_F\bigr] (1 + 2\sqrt{e\log d}) \over \sqrt{n\gl}}  +  { 12 \over \sqrt{n}}. \end{array}\endeqn
\end{example}
\begin{proof} The estimation of $X_\ast$ is straightforward and we estimate $R_n$  as follows.  For any $q>1$, there holds
\begeqn\label{eq:inter-21} \begin{array}{ll}   {R}_n & =  {1\over \lfloor{n\over
2}\rfloor} \EX_\bz\EX_\gs\Bigl\| \sum_{i=1}^{\lfloor{n\over
2}\rfloor}\gs_i X_{i ({\lfloor{n\over
2}\rfloor+i})}\Bigr\|_{(2,\infty)} \\ & =    {1\over \lfloor{n\over
2}\rfloor} \EX_\bz\EX_\gs \sup_{\ell\in \N_d} \left(   \sum_{k\in \N_d} \bigl| \sum_{i=1}^{\lfloor{n\over
2}\rfloor}\gs_i (x^k_i- x^k_{{\lfloor{n\over
2}\rfloor}+i})(x^\ell_i- x^\ell_{{\lfloor{n\over
2}\rfloor}+i})  \bigr|^2\right)^{1\over 2}\\
&  \le {1\over \lfloor{n\over
2}\rfloor} \EX_\bz  \left( \sum_{k\in \N_d} \EX_\gs \sup_{\ell\in \N_d}   \bigl| \sum_{i=1}^{\lfloor{n\over
2}\rfloor}\gs_i (x^k_i- x^k_{{\lfloor{n\over
2}\rfloor}+i})(x^\ell_i- x^\ell_{{\lfloor{n\over
2}\rfloor}+i})  \bigr|^2\right)^{1\over 2}.
\end{array}
\endeqn
It remains to estimate the terms inside the parenthesis on the right-hand side of the above inequality.  To this end, we observe, for any $q'>1$,  that
$$\begin{array}{ll} & \EX_\gs \sup_{\ell\in \N_d}   \bigl| \sum_{i=1}^{\lfloor{n\over
2}\rfloor}\gs_i (x^k_i- x^k_{{\lfloor{n\over
2}\rfloor}+i})(x^\ell_i- x^\ell_{{\lfloor{n\over
2}\rfloor}+i})  \bigr|^2 \\ & \le   \EX_\gs \left(\sum_{\ell\in \N_d}  \bigl| \sum_{i=1}^{\lfloor{n\over
2}\rfloor}\gs_i (x^k_i- x^k_{{\lfloor{n\over
2}\rfloor}+i})(x^\ell_i- x^\ell_{{\lfloor{n\over
2}\rfloor}+i})  \bigr|^{2q'}\right)^{1\over q'}\\
&  \le   \left(\sum_{\ell\in \N_d}  \EX_\gs \bigl| \sum_{i=1}^{\lfloor{n\over
2}\rfloor}\gs_i (x^k_i- x^k_{{\lfloor{n\over
2}\rfloor}+i})(x^\ell_i- x^\ell_{{\lfloor{n\over
2}\rfloor}+i})  \bigr|^{2q'}\right)^{1\over q'}.
\end{array}$$
Applying the Khinchin-Kahane inequality (Lemma \ref{lem:khin} in the Appendix) with $q = 2q' =4 \log d$ and $p=2$ to the above inequality yields that
$$\begin{array}{ll} & \EX_\gs \sup_{\ell\in \N_d}   \bigl| \sum_{i=1}^{\lfloor{n\over
2}\rfloor}\gs_i (x^k_i- x^k_{{\lfloor{n\over
2}\rfloor}+i})(x^\ell_i- x^\ell_{{\lfloor{n\over
2}\rfloor}+i})  \bigr|^2 \\ & \le
\left(\sum_{\ell\in \N_d}  (2q')^{q'}   \bigl[\EX_\gs \bigl| \sum_{i=1}^{\lfloor{n\over
2}\rfloor}\gs_i (x^k_i- x^k_{{\lfloor{n\over
2}\rfloor}+i})(x^\ell_i- x^\ell_{{\lfloor{n\over
2}\rfloor}+i})  \bigr|^{2}\bigr]^{q'}\right)^{1\over q'}\\
& =  \left(\sum_{\ell\in \N_d}  (2q')^{q'}    \bigl[\sum_{i=1}^{\lfloor{n\over
2}\rfloor}(x^k_i- x^k_{{\lfloor{n\over
2}\rfloor}+i})^2(x^\ell_i- x^\ell_{{\lfloor{n\over
2}\rfloor}+i})^2   \bigr]^{q'}\right)^{1\over q'}\\
& \le 2q'  \sup_{x,x'\in \X} \|x-x' \|^2_\infty  d^{1\over q'}\bigl[\sum_{i=1}^{\lfloor{n\over
2}\rfloor}(x^k_i- x^k_{{\lfloor{n\over
2}\rfloor}+i})^2\bigr]   \\ & \le 4 e(\log d)   \sup_{x,x'\in \X} \|x-x' \|^2_\infty\bigl[\sum_{i=1}^{\lfloor{n\over
2}\rfloor}(x^k_i- x^k_{{\lfloor{n\over
2}\rfloor}+i})^2\bigr]
\end{array}$$
Putting the above estimation back into (\ref{eq:inter-21}) implies that
$$ \begin{array}{ll}R_n & \le  {\sqrt{4 e\log d} \bigl[\sup_{x,x'\in \X} \|x-x' \|_\infty\bigr]\EX_\bz  \left(\sum_{i=1}^{\lfloor{n\over
2}\rfloor}\|x_i - x_{{\lfloor{n\over
2}\rfloor}+i}\|_F^2\right)^{1\over 2}\big /{\lfloor{n\over
2}\rfloor}}\\ &\le {\sqrt{4 e\log d} \bigl[\sup_{x,x'\in \X} \|x-x' \|_\infty\bigr]\bigl[\sup_{x,x'\in \X} \|x-x' \|_F\bigr]\big / \sqrt{\lfloor{n\over
2}\rfloor}} \\ & \le {4\sqrt{ e\log d}\bigl[\sup_{x,x'\in \X} \|x-x' \|_\infty\bigr]\bigl[\sup_{x,x'\in \X} \|x-x' \|_F\bigr]  \big / \sqrt{n}}. \end{array}$$
Combining this with Theorem \ref{thm:main} implies the inequality (\ref{eq:L21-bnd}). This completes the proof of the example. \end{proof}

In the Frobenius-norm case, the main term of the  bound (\ref{eq:fro-bnd}) is  $\O\bigl( {\sup_{x,x'\in \X }  \|x-x'\|_F^{2}  \over \sqrt{n  \gl}}\bigr)$.   This bound is consistent with that given by \cite{Jin} where $\sup_{x\in \X }  \|x\|_F$ is assumed to   bounded by some constant $B$.   Comparing the generalisation bounds in the above examples. The key terms  $X_\ast$ and $R_n$ mainly differ in two quantities, i.e. $\sup_{x,x'\in \X} \|x-x'\|_F$ and $\sup_{x,x'\in \X} \|x-x'\|_\infty.$  We argue that $\sup_{x,x'\in \X} \|x-x'\|_\infty$ can be much less than $\sup_{x,x'\in \X} \|x-x'\|_F.$  For instance,  consider the input space $\X = [0,1]^d.$  It is easy to see that $\sup_{x,x'\in \X} \|x-x'\|_F = \sqrt{d}$ while $\sup_{x,x'\in \X} \|x-x'\|_\infty \equiv 1.$ Consequently, we can summarise the estimations as follows:

\begin{itemize}

\item {\bf Frobenius-norm}:  $X_\ast =  d, ~~ \hbox{ and }~~ R_n \le  {2 d  \over \sqrt{n}}.$

\item {\bf Sparse $L^1$-norm}:     $ X_\ast =  1, ~~ \hbox{ and }~~ R_n \le  {4 \sqrt{e \log d}\over \sqrt{n}}.$

\item {\bf Mixed $(2,1)$-norm}:   $ X_\ast =  \sqrt{d}, ~~ \hbox{ and }~~ R_n \le  {4 \sqrt{e d\log d}\over \sqrt{n}}.$

\end{itemize}
Therefore, when $d$ is large, the generalisation bound  with sparse $L^1$-norm regularisation is much better than that with Frobenius-norm regularisation while the bound with mixed $(2,1)$-norm are between the above two. These theoretical results are nicely consistent with the rationale that sparse methods are more effective in dealing with high-dimensional data.

We end this section with two remarks. Firstly, in the setting of trace-norm regularisation, it remains a question to us on how to establish more accurate estimation of $R_n$ by using the Khinchin-Kahane inequality.   Secondly, the bounds in the above examples are true for similarity learning with different matrix-norm regularisation. Indeed, the generalisation bound for similarity learning in Theorem \ref{thm:sim-main} tells us that it suffices to estimate $\wX_\ast$ and $\wR_n$. In analogy to the arguments in the above examples, we can get the following results.  For similarity learning formulation (\ref{eq:sim-model}) with Frobenius-norm regularisation, there holds  $$\wX_\ast  = \sup_{x\in \X} \|x\|^2_F, \qquad~ \wR_n \le  {2\sup_{x} \|x\|^2_F \over \sqrt{n}}.$$ For $L^1$-norm regularisation, we have  $$\wX_\ast  = \sup_{x\in \X} \|x\|^2_\infty, \qquad ~ \wR_n \le  {4\sup_{x \in \X} \|x\|^2_\infty \sqrt{e\log d }\big/\sqrt{n}}.$$ In the setting of $(2,1)$-norm,  we obtain $$\wX_\ast  = \sup_{x\in \X} \|x\|_\infty\sup_{x\in \X} \|x\|_F, \qquad ~  \wR_n \le  {4\bigl[\sup_{x\in \X} \|x\|_F \sup_{x\in \X} \|x\|_\infty\bigr]  \sqrt{e\log d}\big/ \sqrt{n}}.$$
Putting these estimations back into Theorem \ref{thm:sim-main} yields generalisation bounds for similarity learning with  different matrix norms. For simplicity, we omit the details here.

\section{Conclusion and Discussion}\label{sec:conclusion}
In this paper we are mainly concerned with theoretical generalisation analysis of the regularized metric and similarity  learning. In particular, we first showed that the generalisation analysis for metric/similarity learning reduces to the estimation of the Rademacher average over ``sums-of-i.i.d." sample-blocks.  Then, we derived their generalisation bounds with different matrix regularisation terms.  Our analysis indicates that sparse  metric/similarity learning with $L^1$-norm regularisation could lead significantly better bounds than that with the Frobenius norm regularisation, especially when the dimension of the input data is high.   Our novel generalisation analysis  develops the techniques of U-statistics \cite{Gine,Clem} and Rademacher complexity analysis \cite{BM,KPan}.  Below we mention several questions remaining to be further studied.

Firstly, in Section \ref{sec:bounds}, the derived bounds for metric and similarity learning with trace-norm regularisation were the same as those with Frobenius-norm regularisation.  It would be very interesting to derive the bounds similar to those with sparse $\ell^1$-norm regularisation. The key issue is to estimate the Rademacher complexity term (\ref{eq:rad-emp}) related to the spectral norm using the Khinchin-Kahne inequality. However, we are not aware of such Khinchin-Kahne inequalities for general matrix spectral norms. Another alternative is to apply the advanced oracle inequalities in \cite{Koltchinskii}.

Secondly,  this study only investigated the generalisation bounds for metric and similarity learning. We can further get the consistency estimation under strong assumptions on the loss function and underlying distribution.  Specifically, we assume that the loss function is the least square loss, the matrix norm is the Frobenius norm and the bias term $b$ is fixed to be zero.  In addition, assume that the true minimizer $M_\ast = \arg\min_{M\in \mathbb{S}^d} \E(M,0)$ exists and let $M_{\bz} =\arg\min_{M\in \mathbb{S}^d} \bigl[\E_{\bf z}(M,0) + \gl \|M\|_F^2 \bigr].$ Observe that \begin{equation}\label{eq:1} \begin{array}{ll} \E(M_\bz,0)- \E(M_\ast,0) & = \iint  \langle M_{\bz}-M_\ast,  x (x')^T \rangle^2 d\rho(x)\rho(x')\\ &  = \langle \C(M_{\bz}-M_\ast), M_{\bz}-M_\ast\rangle,\end{array}\end{equation}
where $\C = \iint (x(x')^T)\otimes   (x (x')^T) d\rho(x)\rho(x')$  and $\otimes$ represents the tensor product of matrices. Equation (\ref{eq:1}) implies that $\E(M_\bz,0)- \E(M_\ast,0) = \iint  \langle M_\bz-M_\ast,  x (x')^T \rangle^2 d\rho(x)\rho(x') \ge \lambda_{\min} (\C) \|M_\bz - M_\ast\|_F^2,$ where $\lambda_{\min}(\C)$ is the minimum eigenvalue of the $d^2\times d^2$ matrix $\C.$  Furthermore, observe that $\E(M_\bz,0)- \E(M_\ast,0)$ is further bounded by
\begeqn\begin{array}{ll} & \bigl[\E(M_\bz,0) - \E_\bz(M_\bz,0) \bigr]   + \bigl[\E_\bz(M_\bz,0) + \gl \|M_{\bz}\|_F^2\bigr] - \E(M_\ast,0)\\
& \le \bigl[\E(M_\bz,0) - \E_\bz(M_\bz,0) \bigr]  + \bigl[\E_\bz(M_\ast,0) + \gl \|M_{\ast}\|_F^2\bigr] - \E(M_\ast,0) \\
& = \bigl[\E(M_\bz,0) - \E_\bz(M_\bz,0) \bigr] + \bigl[\E_\bz(M_\ast,0) - \E(M_\ast,0)\bigr] + \gl \|M_{\ast}\|_F^2, \end{array}\endeqn where the inequality follows from the definition of the minimizer $M_\bz.$
Combining equation (\ref{eq:1}) with the above estimation together implies that
\begeqn\label{eq:2} \begin{array}{ll} \lambda_{\min} (\C) \|M_\bz - M_\ast\|_F^2 & \le   \bigl[\E(M_\bz,0) - \E_\bz(M_\bz,0) \bigr] \\ & + \bigl[\E_\bz(M_\ast,0) - \E(M_\ast,0)\bigr] + \gl \|M_{\ast}\|_F^2 .\end{array}\endeqn
Using a similar argument as that for proving Theorem \ref{thm:main} and Example \ref{exm:fro},  we can get that $\bigl[\E(M_\bz,0) - \E_\bz(M_\bz,0) \bigr]+ \bigl[\E_\bz(M_\ast,0) - \E(M_\ast,0)\bigr] \le  {C \ln({2\over \gd})\over \gl \sqrt{n}}$ with a high confidence $1-\gd,$ where the constant $C$ does not depend on $\bz.$  Consequently,  putting this estimation with inequality (\ref{eq:2}) together implies that $\|M_\bz - M_\ast\|_F^2 \le {1 \over \lambda_{\min} (\C)}\Big[ C {\ln({2\over \gd})\over \gl \sqrt{n}} + \gl \|M_\ast\|_F^2\Big].$ Choosing $\gl = n^{-{1\over 4}} $ yields the consistency estimation:
$$\|M_\bz - M_\ast\|_F^2 \le  { {C\ln({2\over \gd})} + \|M_\ast\|_F^2\over \lambda_{\min} (\C)\; n^{1\over 4}}.$$
For the hinge loss, equality (\ref{eq:1}) does not hold true any more.  Hence, it remains a question on how to get the consistency estimation for metric and similarity learning with general loss functions.

Thirdly, in many applications involving multi-media data, different aspects of the data may lead to several different, and apparently equally valid notions of similarity. This leads to a natural question to combine multiple similarities and metrics for a unified data representation. An extension of multiple kernel learning approach was proposed in \cite{ML} to address this issue.  It would be very interesting to investigate the theoretical generalisation analysis  for this  multi-modal similarity learning framework.  A possible starting point would be the techniques established for learning the kernel problem \cite{Ying-COLT,Ying-NECO}.

Finally, the target of supervised metric learning is to improve the generalisation performance of kNN classifiers. It remains a challenging question to investigate how the generalisation performance of kNN classifiers relates to the generalisation bounds of metric learning given here.

\section*{Acknowledgement:}{We are grateful to the referees for their constructive comments and suggestions. This work is supported by the EPSRC under grant EP/J001384/1. The corresponding author is Yiming Ying.}

\section*{Appendix}

In this appendix we assemble some facts, which were used to  establish
generalisation bounds for metric/similarity learning.

\begin{definition} We say the function $f:\displaystyle\dprod_{k=1}^n \Omega_k \rightarrow
\mathbb{R}$ with bounded differences $\{ c_k \}_{k=1}^n $ if, for
all $1\le k \le n $, $$\begin{array}{ll} \dmax_{z_1,\cdots
,z_k,z^\prime_{k}\cdots ,z_n} & |  f(z_1, \cdots
,z_{k-1},z_k,z_{k+1}, \cdots , z_n)\\ & -f(z_1, \cdots
,z_{k-1},z^\prime_k,z_{k+1}, \cdots , z_n) |\le c_k\end{array}
$$
\end{definition}

\begin{lemma} (McDiarmid's inequality \cite{McD}) Suppose $f:\displaystyle\dprod_{k=1}^n \Omega_k \rightarrow
\mathbb{R}$ with bounded differences $ \{ c_k \}_{k=1}^n $ then ,
for all $\epsilon>0 $, there holds
$$ {\bf Pr}_{\bf z} \biggl\{ f(\bz) -\mathbb{E}_{\bf z}f({\bf z})\ge \epsilon\biggr\}\le
 e^{- \frac{2\epsilon^2}{\sum_{k=1}^n c_k^2} }.$$
\label{lem:McD} \end{lemma}

Finally we list a useful property for U-statistics. Given the i.i.d. random variables $z_1,z_2,\ldots, z_n\in \Z,$
let $q: Z\times Z \to \R$ be a symmetric real-valued function.
Denote  a U-statistic of order two  by $U_n =  {1\over n(n-1)}
\dsum_{i\neq j} q(x_i,x_j).$ Then, the U-statistic $U_n$ can be
expressed as \begeqn\label{eq:U-stat-rep} U_n = {1\over n!}
\dsum_{\pi} {1\over \lfloor{n\over 2}\rfloor}
\dsum_{i=1}^{\lfloor{n\over 2}\rfloor}
q(z_{\pi(i)},z_{\pi(\lfloor{n\over 2}\rfloor+i)})
\endeqn
where the sum is taken over all permutations $\pi$ of
$\{1,2,\ldots,n\}.$  The main idea underlying this representation is
to reduce the analysis to the ordinary case of i.i.d. random
variable blocks. Based on the above representation, we can prove  the
following lemma which plays a critical role in deriving
generalisation bounds for metric learning. For completeness, we
include a proof here.  For more details on U-statistics, one is referred to \cite{Clem,Gine}.

\beg{lemma} \label{lem:U-stat-lem} Let $q_\tau: \Z\times \Z  \to \R$
be real-valued functions indexed by $\tau\in \T$ where $\T$ is some
index set.  If $z_1, \ldots, z_n$ are i.i.d. then we have that
$$ \EX \Bigl[\dsup_{\tau\in \T}  {1\over n(n-1)}\dsum_{i\neq j} q_\tau(z_i,z_j)\Bigr] \le \EX
 \Bigl[\dsup_{\tau\in \T} {1\over \lfloor{n\over 2}\rfloor}
\dsum_{i=1}^{\lfloor{n\over 2}\rfloor}
q_\tau(z_{i},z_{\lfloor{n\over 2}\rfloor+i})\Bigr].$$
\end{lemma}

\begin{proof}  From the representation of U-statistics (\ref{eq:U-stat-rep}),  we observe that
$$\beg{array}{ll}  \EX \Bigl[\dsup_{\tau\in \T}  {1\over n(n-1)}\dsum_{i\neq j}
q_\tau(z_i,z_j)\Bigr] & =  \EX\dsup_\tau {1\over n!} \dsum_{\pi}
{1\over \lfloor{n\over 2}\rfloor} \dsum_{i=1}^{\lfloor{n\over
2}\rfloor}
q_\tau(z_{\pi(i)},z_{\pi(\lfloor{n\over 2}\rfloor+i)}) \\
&\le {1\over n!}  \EX \dsum_{\pi} \dsup_\tau {1\over \lfloor{n\over
2}\rfloor} \dsum_{i=1}^{\lfloor{n\over 2}\rfloor}
q_\tau(z_{\pi(i)},z_{\pi(\lfloor{n\over 2}\rfloor+i)})  \\
& = {1\over n!}  \dsum_{\pi} \EX  \dsup_\tau {1\over \lfloor{n\over
2}\rfloor} \dsum_{i=1}^{\lfloor{n\over 2}\rfloor}
q_\tau(z_{\pi(i)},z_{\pi(\lfloor{n\over 2}\rfloor+i)})\\
& =\EX
 \Bigl[\dsup_{\tau\in \T} {1\over \lfloor{n\over 2}\rfloor}
\dsum_{i=1}^{\lfloor{n\over 2}\rfloor}
q_\tau(z_{i},z_{\lfloor{n\over 2}\rfloor+i})\Bigr].
\end{array}$$
This  completes the proof of the lemma.
\end{proof}

We need the following contraction  property of the Rademacher
averages which  is essentially implied by Theorem 4.12 in Ledoux and
Talagrand \cite{LT},  see also \cite{BM,KPan}.

\begin{lemma}\label{lem:contr-prop}Let $F$ be a class of uniformly bounded real-valued
functions on $(\gO,\mu)$ and $m\in\N$. If for each $i\in \{1,
\ldots, m\}$, $\Psi_i: \R \to \R$ is a function with $\Psi_i(0)=0$
having a Lipschitz constant $c_i$, then for any $\{x_i\}_{i=1}^m$,
\begeqn \EX_\epsilon\Big( \dsup_{f\in F} \big|\dsum_{i=1}^m
\epsilon_i \Psi_i(f(x_i)) \big|\Big)
 \le 2   \EX_\epsilon\Big(  \dsup_{f\in F}
 \Big|\dsum_{i=1}^m c_i\epsilon_i
f(x_i) \big| \Big).\label{contr}\endeqn
\end{lemma}

The last property of Rademacher averages is the Khinchin-Kahne inequality (see e.g. \cite[Theorem 1.3.1]{Gine}).

\begin{lemma}\label{lem:khin} For $n\in \N$,  let $\{f_i \in \R:  i\in \N_n\} $, and $\{\gs_i: i\in \N_n\}$ be a family of i.i.d. Rademacher variables.  Then, for any $1<p<q<\infty$ we have
$$ \left(\EX_\gs  \bigl|\sum_{i\in \N_n}\gs_i f_i \bigr|^q\right)^{1\over q} \le  \left( {q-1 \over p-1}\right)^{1\over 2} \left(\EX_\gs  \bigl|\sum_{i\in \N_n}\gs_i f_i \bigr|^p\right)^{1\over p}  $$

\end{lemma}

\end{document}